%% file: main.tex
\pdfoutput=1
\documentclass[11pt,a4paper]{article}
\usepackage[hyperref]{emnlp2020}
\usepackage{times}
\usepackage{latexsym}

\usepackage{microtype}

\aclfinalcopy 

\usepackage{graphicx, amsmath, amssymb, amsthm, bbm, mathtools, nicefrac, algorithm, algorithmic, float, txfonts, centernot, xspace}

\newtheorem{defn}{Definition}[section]
\newtheorem{pro}{Proposition}[section]
\newtheorem{thm}{Theorem}[section]

\usepackage{tikz}
\usetikzlibrary{arrows, automata, positioning}

\newcommand{\Prob}{\mathbb{P}}

\DeclarePairedDelimiter{\sbk}{[}{]}
\DeclarePairedDelimiter{\rbk}{(}{)}

\DeclarePairedDelimiter{\cbk}{\{}{\}}
\DeclarePairedDelimiter{\vbk}{|}{|}

\newcommand{\wtv}{\texttt{word2vec}\xspace}
\newcommand{\weekdays}{\texttt{Weekdays}\xspace}

\title{Word2vec Conjecture and A Limitative Result}

\author{%
  Falcon Z.~Dai \\
  Toyota Technological Institute at Chicago \\
  Chicago, USA \\
  \texttt{dai@ttic.edu}
}

\date{}

\begin{document}
\maketitle
\begin{abstract}
  \input{abstract}
\end{abstract}

\input{introduction}
\input{related_work}
\input{results}
\input{discussion}

\section*{Acknowledgments}
We thank David McAllester for discussions on word2vec.
We thank Zheng Cai for discussions on linguistic analogy.
We thank David Yunis for reading an early version of this work.
Lastly, we thank an anonymous reviewer for constructive feedback and encouragement on an earlier version which led to an improved presentation.

\bibliographystyle{acl_natbib}
\bibliography{references}


\clearpage

\appendix
\input{appendix}

\end{document}

%% file: abstract.tex
Being inspired by the success of \texttt{word2vec} \citep{mikolov2013distributed} in capturing analogies, we study the conjecture that analogical relations can be represented by vector spaces.
Unlike many previous works that focus on the distributional semantic aspect of \texttt{word2vec}, we study the purely \emph{representational} question: can \emph{all} semantic word-word relations be represented by differences (or directions) of vectors?
We call this the word2vec conjecture and point out some of its desirable implications.
However, we will exhibit a class of relations that cannot be represented in this way, thus falsifying the conjecture and establishing a limitative result for the representability of semantic relations by vector spaces over fields of characteristic 0, e.g., real or complex numbers.

%% file: introduction.tex
\section{Introduction}
\label{sec:introduction}

\citet{mikolov2013distributed} have inspired a wave of excitement in natural language processing (NLP) research---much of which is still being felt today---that relies on some notion of distributional semantics \cite{firth1957synopsis} and emphasizes learning from large corpora. As the plain yet memorable name of its codebase \wtv suggests, \citet{mikolov2013distributed} describes a learning-based method to represent words with vectors (in a finite-dimensional inner product vector space).
Furthermore, perhaps most surprisingly, \citet{mikolov2013distributed} discovered that \emph{many} analogies in the form of ``$a$ is to $b$ as $c$ is to $d$'' are captured by the arithmetics in the representation vector space.
Many excellent subsequent works \citep{Pennington_Socher_Manning_2014,Levy_Goldberg_2014b,arora2016latent} have studied this empirical phenomenon and provide some theoretical explanations for the emergence of semantic analogies from a training objective based on estimating distribution of context words.
We, on the other hand, study a purely representational question inspired by, and generalized from the original observation. Can we represent words as vectors and all of their analogical relations as vector differences? This question is purely representational because we are only concerned with if such representation is possible ignoring how to obtain them from corpus statistics. This is a generalization because we ask if this is true for \emph{all} analogical relations. We call this question the word2vec conjecture to signify its origin and representations by vector spaces, but we do not suggest that the original work \cite{mikolov2013distributed} claimed or relied on the conjecture being true.

At the core of the word2vec conjecture, it is a concern whether the vector space structure suffices to capture the structure of semantic relations.
Perhaps not surprisingly to some readers, representing relations as vectors has its limitations.
We will focus on a concrete limitation when trying to represent group-like relations with vectors: the relations can loop around but vectors cannot due to linearity.
We will use, as a simple but illustrative example, the successor relation over weekdays, i.e., ``$a$ is the weekday immediately after $b$'', and the relations generated by it, e.g., ``$a$ is the second weekday after $b$'' (see Section~\ref{sec:weekdays}).
These relations form a cyclic group of order 7 as weekdays repeat themselves every seven days and, in particular, the successor relation is of order 7.
However no nonzero vector over a field of characteristic 0 is of order 7, thus forcing all the relations to be represented as the zero vector (of order 1).

Despite the negative results, this work is primarily motivated by the positive goal of improving algorithmic natural language understanding.
Accordingly, we will first motivate and formulate the problem of representing relations (Section~\ref{sec:nlu}). Then we formalize the vector representation of relations and pose the word2vec conjecture (Section~\ref{sec:conjecture}). We will discuss some desirable implications of the conjecture. Then we will show a limitative result by exhibiting loopy semantic relations that cannot be represented by vector spaces (Section~\ref{sec:order-cyclic}).
Lastly we provide an extended discussion (Section~\ref{sec:discussion}).
In the interest of self-containedness, we review the relevant algebraic concepts in the appendix (Appendix~\ref{app:group-theory} and \ref{app:field-vs}).

Our main contributions are two-fold.
\begin{itemize}
  \item We show that some semantic relations cannot be represented by vector spaces, thus showing a fundamental limitation of \wtv.
  \item By using algebraic arguments to establish the limitative result, we demonstrate the utility of algebraic arguments in the study of semantic representations.
\end{itemize}

%% file: related_work.tex
\section{Related work}
\label{sec:related_work}

Prior to \texttt{word2vec}, \citet{Turney_Littman_2005, turney2010frequency} studied representing words as vectors and relations as matrices. This is not a relation-as-vector representation and, in fact, can represent relations with a group structure (see Section~\ref{sec:solutions}).
The conjecture that we propose and answer is similar in spirit to their concerns on the limitations of representations (see Section 8 \citep{turney2010frequency}).
One major difference between the learning method of \citet{turney2010frequency} and that of \citet{mikolov2013distributed} is characterized as the count vs. prediction issue and is studied extensively \citep{baroni2014don,Pennington_Socher_Manning_2014}.
Basing on the connection to the ratio of conditional distributions noted by \citet{mikolov2013distributed} and by extension, to pointwise mutual information (PMI) \citep{church1990word}, \citet{Levy_Goldberg_2014b} studied building vector representation of relations explicitly. This work provides a limitative result for such efforts, learning or not. We will specialize our arguments to the ratio of conditional distributions, a common characterization of \wtv, to convey more concrete intuitions (see Section~\ref{sec:pmi}).
\citet{arora2016latent} studied rigorously the consequence of prediction under the choice of small dimensionality and provided additional insight on learning word-vector representations by extending a theoretical model of word generation proposed by \citet{mnih2007three}.
We note that due to the similarity measures used by \citet{mikolov2013distributed,arora2016latent}, relations are represented as directions and lines, respectively, instead of vectors. A conscious simplification is made to focus on vectors with the hope to convey the main ideas more smoothly (see Section~\ref{sec:revisit}).

%% file: results.tex
\section{Formalism}
\label{sec:formalism}

\subsection{Words and relations}
\label{sec:nlu}

Short of a persuasive formal definition of linguistic understanding, we will posit a weak consequence.

{\bf Postulate.} If someone \emph{understands} a collection of words, then he or she can correctly answer questions about relations over those words.

These relations can mirror physical relations which is of interest in grounding research but they can arise from other sources. More importantly, regardless of the sources of these relations, we can study the \emph{structure} of them in abstract. Ultimately it is the structure of relations that we wish to adequately represent in some mathematical objects for manipulation and computation. As an example, consider the weekdays. If someone claims to understand the weekdays, then we expect them to correctly answer questions such as ``Does Tuesday follow Monday?'' and ``Is Sunday the second weekday after Friday?''.

{\bf Problem setting.} Formally, given a set of words (concepts) $\mathcal{W}$, and a set of binary relations $\mathcal{R}$ where $r \in \mathcal{R}$ is a subset of $\mathcal{W} \times \mathcal{W}$.\footnote{We will focus on binary relations for the rest of the article and specialize relations to mean binary relations.}
We want to find a representation $\rbk*{\mathcal{W}, \mathcal{R}, \mathcal{E}, \varphi}$
such that $\varphi : \mathcal{R} \rightarrow \mathcal{E}$ nontrivially preserves the structure of $\mathcal{R}$, i.e., a nontrivial \emph{homomorphism}.

\subsection{Example of \weekdays}
\label{sec:weekdays}
Let $\mathcal{W} \coloneqq \cbk*{\text{``Monday''}, \text{``Tuesday''}, \cdots, \text{``Sunday''}}$. Consider the usual successor relation $s$ over the weekdays, e.g., $\rbk*{\text{``Monday''}, \text{``Tuesday''}} \in s$ and $\rbk*{\text{``Monday''}, \text{``Sunday''}} \notin s$.
Consider a relation composition operation $\circ$ that constructs a relation out of two relations. Let $r \circ r' \coloneqq \cbk*{ \rbk*{w, w'} : \exists w''. (w, w'') \in r \text{ and } (w'', w') \in r' }$.
The successor relation generates other relations by composing with itself, e.g., $s^2 = s \circ s$ which intuitively says whether $w'$ is the second weekday after $w$. In this fashion, we can generate relation $s^k = \underbrace{s\circ\cdots\circ s}_{k \text{ times}}$ where $\rbk*{w, w'} \in s^k$ iff $w'$ is the $k$-th weekday after $w$.
However, as weekdays repeat themselves every seven days, we have $s^k = s^\ell$ iff $k \equiv \ell \pmod 7$.
Hence collecting all \emph{powers} of $s$ results in a set $\mathcal{R} \coloneqq \cbk*{ s, s^2, \cdots, s^7 }$, a set of seven binary relations.
In fact, $\mathcal{R}$ is a cyclic group of order 7 (see Appendix~\ref{app:proof-weekdays-z7} for a proof).
\begin{pro}
  \label{pro:weekdays-z7}
  $\mathcal{R} = \cbk*{s, s^2, \cdots, s^7}$ together with composition $\circ$ forms a group isomorphic to $\mathbb{Z}_7$.
\end{pro}




\subsection{Relation-as-vector representations}
\label{sec:rel2vec}

Suppose $(\mathcal{W}, \mathcal{R}, \mathbb{V}, \varphi)$ embeds words to a vector space over the reals $\mathbb{V}$ via the embedding map $\varphi : \mathcal{W} \rightarrow \mathbb{V}$.
Motivated by the empirical observation of emergence of relations as differences of word-vectors \cite{mikolov2013distributed}, we call $(\mathcal{W}, \mathcal{R}, \mathbb{V}, \varphi)$ a \emph{relation-as-vector representation} if for any relation $r \in \mathcal{R}$, we have $\varphi(b) - \varphi(a) = \varphi(d) - \varphi(c)$ for any $a, b, c, d \in \mathcal{W}$ satisfying $(a, b) \in r$ and $(c, d) \in r$.
With a slight abuse of notations, we will denote this relation-vector as $\varphi(r) = \varphi(b) - \varphi(a) \in \mathbb{V}$.\footnote{$\varphi$ is a category theoretic functor in disguise.}
Lastly, we say that a set of relations $\mathcal{R}$ is \emph{well-represented} by $(\mathbb{V}, \mathcal{R}, \varphi)$ if $\varphi(r) \neq \varphi(s)$ for all $r \neq s \in \mathcal{R}$, i.e., we can distinguish distinct relations.

\subsection{Word2vec conjecture}
\label{sec:conjecture}
We ask whether relation-as-vector representation can universally represent any relations.

{\bf Conjecture.} Can \emph{all} binary relations over words be well-represented by some relation-as-vector representation?

\subsection{An implication for metalinguistics}
Besides the potential in applications such as solving analogical queries \citep{mikolov2013distributed}, the recursiveness of representations of both words and relations might inspire interesting questions that are metalinguistic in nature. For example, if our conjecture is true---relations are representable as vectors---, then how about higher-order relations, i.e., relations between relations? Can they all be represented in the same way, i.e., as vectors? It is precisely this possibility that initiated our study of the current problem.

\section{Results}
\label{sec:results}
We will answer the conjecture in the negative in two steps. The first is about algebraic structures, resulting in a general negative result (Theorem~\ref{thm:negative}), and the second is an empirical observation about the existence of certain class of semantic relations.

\subsection{Algebraic aspects of semantics}
Note that by the definition of relation-as-vector representation, $\varphi$ is a homomorphism from $\rbk*{\mathcal{R}, \circ}$ to the abelian group $\rbk*{\mathbb{V}, +}$. Suppose $r, r' \in \mathcal{R}$ and $\varphi(r) = \varphi(c) - \varphi(b)$ and $\varphi(r') = \varphi(b) - \varphi(a)$ for some $a, b, c \in \mathcal{W}$, then $\varphi(r \circ r') = \varphi(c) - \varphi(b) + \varphi(b) - \varphi(a) = \varphi(r) + \varphi(r')$.

This key observation means that in a relation-as-vector representation, we cannot assign two arbitrary vectors to represent two relations if the two relations are compositionally related to each other. In particular, the \emph{identity relation}, i.e., $(w, w') \in e$ iff $w = w'$, has to be assigned the zero vector. And an inverse relation would have the negative relation-vector so they sum to the zero vector.

In order to prove that \emph{no} good representation exists, our strategy is to find that some algebraic invariants, which are preserved by homomorphisms, are not equal in $(\mathcal{R}, \circ)$ and in $(\mathbb{V}, +)$. Therefore no $\varphi$ can well represent the given relations.

\subsection{Orders and cyclic relations}
\label{sec:order-cyclic}
Consider the \emph{order} of a group element (see Appendix~\ref{app:group-theory}). It is a basic fact that the order of its homomorphic image in $H$ divides its order in $G$, i.e., $\vbk*{\varphi(r)}$ divides $\vbk*{r}$ (see Proposition~\ref{pro:order-hom}). This fact is useful because the only element in a vector space over real numbers with a finite order is the zero vector. It has order of one $\vbk*{\mathbf{0}} = 1$.

\begin{thm}
  \label{thm:negative}
  If the given relations contains a relation $r$ with a nontrivial finite order, i.e., $\vbk*{r} \ne 1$ and $\vbk*{r} < \infty$, then they cannot be well-represented by any relation-as-vector representation using a vector space over a field of characteristic 0, e.g., real or complex numbers.
\end{thm}

The key idea is that any homomorphism $\varphi$ would have to represent all powers of $r$ with the zero vector, thus failing to encode them distinctly (see Appendix~\ref{app:proof-negative} for a detailed proof).

Now, it remains to exhibit that there are semantic relations with a nontrivial finite order (whose powers forms cyclic groups). Conveniently, \weekdays provides such an example: the successor relation $s$ is of order 7. Similarly, many temporal concepts exhibit cyclic relations such as $\mathbb{Z}_{12}$ for the months, $\mathbb{Z}_{24}$ for the hours. An interesting non-temporal example are the antonyms, over which a parity relation is of order 2.

\subsection{Specialization to ratios of conditional distributions}
\label{sec:pmi}
Due to the continuous bag of words (CBOW) learning procedure used in \cite{mikolov2013distributed}, it is suggested that the word-vector of $w$ in \wtv is an approximation of the logarithm of the conditional distributions of context words around $w$. Suppose the space of context words is $\mathcal{C}$ and the probability of a context word $c \in \mathcal{C}$ appearing near $w$ is $\Prob\sbk*{c \vert w}$,
then the specific relation-as-vector representation suggested is $\rbk*{\mathcal{W}, \mathcal{R}, \mathbb{V}\rbk*{\mathcal{C}}, \psi}$
where $\mathbb{V}\rbk*{\mathcal{C}}$ is the \emph{free} real vector space generated by $\mathcal{C}$
and $\psi : \mathcal{W} \rightarrow \mathbb{V}\rbk*{\mathcal{C}} : w \mapsto \rbk*{c \mapsto \log \Prob\sbk*{c \vert w}}$.
Then a relation $r \in \mathcal{R}$ is represented by the logarithm of the ratio of conditional distributions over each context word
\begin{align*}
  \psi(r) &= \psi(w') - \psi(w) \\
  &= \rbk*{c \mapsto \log \Prob\sbk*{c \vert w'} - \log \Prob\sbk*{c \vert w}} \\
  &= \rbk*{c \mapsto \log \frac{\Prob\sbk*{c \vert w'}}{\Prob\sbk*{c \vert w}}}
\end{align*}
where $(w, w') \in r$.

Note that the general conclusion of Theorem~\ref{thm:negative} applies to this specific relation-as-vector representation: $\psi$ cannot well-represent a loopy relation.
However, we can specialize the analysis to $\psi$ in the hope of gaining more intuition.

Suppose over some context word $c$, some non-identity relation $r$ has a nonzero coordinate, i.e., $\frac{\Prob\sbk*{c \vert w'}}{\Prob\sbk*{c \vert w}} \ne 1$, otherwise $r$ is the identity relation and $\psi(r) = \mathbf{0}$.
Any positive power of $r$, $r^k$ for $k > 0$, would have a nonzero coordinate at $c$ as $\rbk*{\frac{\Prob\sbk*{c \vert w'}}{\Prob\sbk*{c \vert w}}}^k \ne 1$.
Therefore $\psi\rbk*{r^k} \ne \mathbf{0}$ and $\psi(r)$ cannot have a finite order.
Intuitively, the ratio of distributions of context words is monotonic with respect to the power and thus cannot wrap around.

%% file: discussion.tex
\section{Discussion}
\label{sec:discussion}

\subsection{Other limitations and examples}
\label{sec:other-limits}
Another representational limitation of vector space, being an abelian group, is that it cannot well-represent non-commutative relations. For example, consider longitude-latitude locations on the globe and translational relations. Due to the spherical geometry, going eastward for 10 miles then going northward for 10 miles is not the same as going north first then east.

The general deficiency of the vector space structure to represent other algebraic structures described in this work applies to other ``X2vec'' methods (some outside of natural language processing).

\subsection{Relevance to learning}
\label{sec:learning}
Broadly speaking in learning problems, there is usually some target $x^*$ within some hypothesis class to be learned from data, and the learned hypothesis $\hat{x}$ is approximately correct due to the limited data or computation. In this perspective, what we show with Theorem~\ref{thm:negative} is that certain semantic relations does not correspond to any $x^*$ within relation-as-vector representations. Practically for natural language understanding applications, we have to consider whether the representation class (and algorithms) is structurally compatible with the class of semantic relations.

\subsection{Possible solutions}
\label{sec:solutions}
Based on Cayley's theorem and basic representation theory, invertible square matrices would well represent relations exhibiting group structures. Less generally, for cyclic relations exhibiting $\mathbb{Z}_k$ structure, we can represent the successor relation with the scalar multiplication by complex number $e^{2\pi i/ k}$. These solutions are not relation-as-vector representations.

\subsection{Revisiting \wtv}
\label{sec:revisit}
An expert reader might point out that due to the cosine similarity used in \wtv, a relation is represented by a direction (and in the case of \citet{arora2016latent}, a line due to squaring). However, this difference does not break our cyclic examples as directions or lines can be thought as quotient vector spaces where vectors of positive multiple of each other are identified, or those of nonzero multiple, respectively.

\subsection{Analogy and abstract algebra}

Analogies and metaphors are important linguistic phenomena that relate one domain of concepts and relations to another in a way that preserves (some of) its structure. Some researchers argue that it is a key mechanism in linguistic understanding \cite{Falkenhainer_Forbus_Gentner_1989,hofstadter1995copycat,Lakoff_Johnson_2008}. Technically one can model the mechanism of analogy as algebraic structure-preserving maps, i.e., homomorphisms.

This work suggests an algebraic theory of semantics would be helpful, perhaps necessary, for studying structures in semantic representations. We leave the problem of learning (discovering) group-like semantic relations to future work. Note that identifying two groups (known as the group isomorphism problem) are undecidable in general.

%% file: appendix.tex
\section{Elementary group theory}
\label{app:group-theory}
\begin{defn}[Group]
  A \emph{group} is a set $G$ together with a binary operation (group multiplication) $\cdot : G \times G \rightarrow G$ that satisfies the following axioms. As a convention, we write $gh = g \cdot h$ when there is no ambiguity.
  \begin{itemize}
    \item (identity) There is an element $e \in G$ such that for any $x \in G$, $x e = x$. We denote the (right-)identity with $e$ as the left- and right-identity are the same and unique.
    \item (associativity) For any $x, y, z \in G$, we have $(x y) z = x (y z)$. So conveniently, we can write their product as $x y z$ without ambiguity.
    \item (inverses) For any $x \in G$, there is some $y \in G$ such that $x y = e$. We denote the (right-)inverse of $x$ as $x^{-1} = y$ as $y$ is unique and the left- and right-inverses are the same for $x$.
  \end{itemize}
\end{defn}

Similarly to a group, a \emph{monoid} might not have inverses, and a \emph{semigroup} might not have an identity or inverses. Group multiplication is generally not commutative, and we say a group is \emph{abelian} if its group multiplication is commutative, i.e., $a b = b a$. Furthermore, we commonly denote the group multiplication with addition to emphasize its commutativity, e.g., integers with the usual addition $(\mathbb{Z}, +)$ forms an abelian group.

The \emph{order} of a group element $x \in G$ is the smallest positive power of $x$ that is identity, denoted by $\vbk*{x} \coloneqq \min \cbk*{n \in \mathbb{Z}_{> 0} : x^n = \underbrace{x \cdots x}_{n\text{ times}} = e}$. If no positive power of $x$ is equal to $e$, then we say that the order of $x$ is infinite and denote $\vbk*{x} = \infty$.

\begin{defn}[Group homomorphism]
  Given two groups $(G, \cdot_G)$ and $(H, \cdot_H)$, we call a map $\varphi : G \rightarrow H$ a \emph{(group) homomorphism} if for any $a, b \in G$
  $$\varphi\rbk*{a \cdot_G b} = \varphi(a) \cdot_H \varphi(b).$$
\end{defn}
Furthermore, if $\varphi$ is bijective then we say $\varphi$ is an \emph{isomorphism} and $G$ is isomorphic to $G$, denoted by $G \cong H$.

\begin{pro}
  \label{pro:order-hom}
  Given two groups $G$ and $H$ and a homomorphism $\varphi : G \rightarrow H$ between them, for any $g \in G$, we have $\vbk*{\varphi\rbk*{g}}$ divides $\vbk*{g}$.
\end{pro}
\begin{proof}
  It is not hard to see that a homomorphism maps the identity in $G$ to the identity in $H$. Suppose $\vbk*{g} = n$ then $\varphi\rbk*{g^n} = \varphi\rbk*{e_G} = e_H = \varphi(g)^n$. This implies that $\vbk*{\varphi(g)}$ divides $n$.
\end{proof}

If $\varphi$ is an isomorphism, then $\vbk*{\varphi\rbk*{g}} = \vbk*{g}$.

\section{Fields and vector spaces}
\label{app:field-vs}
A \emph{field} $F$ is a set equipped with addition (with identity denoted by $0$) and multiplication (with identity denoted as $1$) and multiplicative inverses exist for nonzero elements. Furthermore, both addition and multiplication are commutative and they follow the distributive law. Common examples are the rationals $\mathbb{Q}$, the reals $\mathbb{R}$, the complex $\mathbb{C}$ and finite fields of prime order $\mathbb{Z}/p \mathbb{Z}$. Note that the integers $\mathbb{Z}$ is not a field due to the missing multiplicative inverses.

The \emph{characteristic} of a field $F$ is defined to be the smallest integer $k$ such that $\underbrace{1 + \cdots + 1}_{k \text{ times}} = 0$ and if no such $k$ exists, then we say that $F$ has characteristic $0$. For example, $\mathbb{R}$ has characteristic $0$, and $\mathbb{Z}/p \mathbb{Z}$ has characteristic $p$.

A \emph{vector space} $\mathbb{V}$ over a field $F$ (also called a \emph{left $F$-module}) consists of an abelian group $(\mathbb{V}, +)$ with the identity element denoted as the zero vector $\mathbf{0}$ and a \emph{scalar multiplication} $\cdot : F \times \mathbb{V} \rightarrow \mathbb{V}$ that plays nicely with addition and multiplication of $F$.

A \emph{free $F$-vector space} generated over a set $S$ is a $F$-vector space with $S$ as its basis. Thus its dimensionality is the size of $S$. Generally speaking, a \emph{free} object is a ``generic'' object containing its generating set that is free of any additional conditions and satisfying only those in the definition of that class of objects.


\section{Proof of Proposition~\ref{pro:weekdays-z7}}
\label{app:proof-weekdays-z7}
\begin{proof}
  First, we check that the relation composition operation $\circ$ is associative over the successor relation $s$.
  If $(w, w') \in (s \circ s) \circ s$, then by definition, there is some $u \in \mathcal{W}$ such that $(w, u) \in s \circ s$ and $(u, w') \in s$. Again by definition, there is some $v \in \mathcal{W}$ such that $(w, v) \in s$ and $(v, u) \in s$. This implies that $(v, w') \in s \circ s$ and $(w, w') \in s \circ (s \circ s)$. Similarly, we can show the reverse.
  Hence, $(w, w') \in (s \circ s) \circ s$ iff $(w, w') \in s \circ (s \circ s)$, and we have $(s \circ s) \circ s = s \circ (s \circ s)$. This justifies the notation for powers of $s$ as $s^k$ in Section~\ref{sec:weekdays}.

  Second, $s^7 \in \mathcal{R}$ is the identity relation, i.e., $(w, w') \in s^7$ iff $w = w'$ since weekdays repeat themselves every seven days. And for any binary relation $r$, $(w, w') \in r$ iff $(w, w') \in r \circ s^7$, so $r \circ s^7 = r$. Moreover, since $r = s^k$ for some $k$, we have $s^k = s^\ell$ iff $k \equiv \ell \pmod 7$. Accordingly, we can write the identity as $s^0 = s^7$.

  Third, we have inverses. For any $s^k$, choose $\ell$ so that $k + \ell \equiv 7 \pmod 7$, then $s^k \circ s^\ell = s^{k + \ell} = s^0$.

  Fourth, we show that \weekdays is isomorphic to $\mathbb{Z}_7 = (\mathbb{Z}/7\mathbb{Z}, +)$. Consider a map $\varphi : \weekdays \rightarrow \mathbb{Z}_7 : s^k \mapsto k \pmod 7$. This is a homomorphism as $\varphi\rbk*{s^k \circ s^\ell} = \varphi\rbk*{s^{k+\ell}} = k + \ell = \varphi\rbk*{s^k} + \varphi\rbk*{s^\ell}$.
  $\varphi$ is clearly injective, and surjective as $\mathbb{Z}_7$ also has 7 elements. Therefore $\varphi$ is an isomorphism.
\end{proof}

\section{Proof of Theorem~\ref{thm:negative}}
\label{app:proof-negative}
\begin{proof}
  First, $F$-vector space $\mathbb{V}$ with addition is an abelian group. The identity is the zero vector denoted by $\mathbf{0}$ whose order is $1$. If $F$ has characteristic $0$, then by definition, for any nonzero vector $\mathbf{v} \in \mathbb{V}$ and any $k > 0$ we have
  $$ (\underbrace{1 + \cdots + 1}_{k \text{ times}}) \mathbf{v} = \underbrace{\mathbf{v} + \cdots + \mathbf{v}}_{k \text{ times}} \ne \mathbf{0}.$$

  Thus any nonzero vector in $\mathbb{V}$ has order infinite. Together with Proposition~\ref{pro:order-hom}, any homomorphism $\varphi$ can only map a relation with nontrivial finite order to the zero vector, i.e., $\varphi(r) = \mathbf{0}$.
  But this implies that all powers are mapped to the zero vector
  $$ \varphi\rbk*{r^k} = \underbrace{ \varphi\rbk*{r} + \cdots + \varphi\rbk*{r} }_{k \text{ times}} = \underbrace{ \mathbf{0} + \cdots + \mathbf{0} }_{k \text{ times}} = \mathbf{0}. $$
  Therefore $\varphi$ do not distinguish the powers of $r$.
\end{proof}